\documentclass[letterpaper, 10 pt, conference]{ieeeconf}  

\IEEEoverridecommandlockouts                              
\usepackage{times} 
\usepackage{amsmath} 
\usepackage{amssymb}  

\title{\LARGE \bf
Representing Robot Task Plans as \\
Robust Logical-Dynamical Systems
}

\author{Chris Paxton$^{1}$, Nathan Ratliff$^{1}$, Clemens Eppner$^{1}$, Dieter Fox$^{1}$
\thanks{$^{1}$NVIDIA, USA
        {\tt\small \{cpaxton, nratliff, ceppner, dieterf\} at nvidia.com}}%
}

\usepackage[utf8]{inputenc}
\usepackage{hyperref}
\usepackage{booktabs}
\usepackage{multirow}
\usepackage{subcaption} 
\usepackage{array}
\usepackage{xspace}

\usepackage{algorithm}
\usepackage{algpseudocode}
\usepackage{graphicx}

\usepackage{amsthm}

\DeclareMathOperator*{\argmin}{arg\,min}

\newtheorem{theorem}{Theorem}

\hypersetup{urlcolor=blue, colorlinks=true}

\begin{document}

\newcommand{\cpax}[1]{\textcolor{red}{#1}}

\maketitle
\thispagestyle{empty}
\pagestyle{empty}

\begin{abstract}
It is difficult to create robust, reusable, and reactive behaviors for robots that can be easily extended and combined. Frameworks such as Behavior Trees are flexible but difficult to characterize, especially when designing reactions and recovery behaviors to consistently converge to a desired goal condition. We propose a framework which we call Robust Logical-Dynamical Systems (RLDS), which combines the advantages of task representations like behavior trees with theoretical guarantees on performance. RLDS can also be constructed automatically from simple sequential task plans and will still achieve robust, reactive behavior in dynamic real-world environments. 
In this work, we describe both our proposed framework and a case study on a simple household manipulation task, with examples for how specific pieces can be implemented to achieve robust behavior.
Finally, we show how in the context of these manipulation tasks, a combination of an RLDS with planning can achieve better results under adversarial conditions.
\end{abstract}

\section{Introduction}
For robots to solve real problems in unstructured dynamic settings, they must be able to intelligently execute tasks that consist of many interdependent steps. In addition, they must be able to react to changing circumstances, falling back or retrying steps when needed, to ensure that they consistently arrive at the correct goal state despite perturbations, environmental changes, or uncertainty resolution. This means that real-world systems will have a complex interacting set of skills that must be used at the appropriate time to achieve a goal. 

The most common way to build complex behavior is via either manual scripting or hierarchical finite state machines~\cite{bohren2010smach}. These systems, unfortunately, quickly grow in complexity and become difficult to expand and maintain. 
Fallback or recovery behaviors must be programmed manually, requiring substantial engineering work to create truly robust and reactive behavior.
Behavior Trees (BTs) address many of these issues around reactivity and ease of use~\cite{paxton2017costar,paxton2018evaluating,colledanchise2018behavior}. However, BTs have complex internal structure that may produce unexpected effects, and the resulting behavior is often not transparent or easily verifiable. Program flow in a BT is determined by a number of complex operations such as sequence, selector, and decorator nodes~\cite{paxton2017costar,colledanchise2018behavior}.

We propose an alternative system that innately enables fallback behaviors, resulting in quick reactions to changing perceptual inputs without the typical explosion of interconnections that one would observe when designing hierarchical FSMs. 
Our Robust Logical-Dynamical Systems~(RLDS) abstract out the internal structure of these models.
More broadly, RLDS are a type of reactive program which can be automatically constructed from a list of operators with specified preconditions and effects. Logical constraints are propagated through the RLDS after construction to ensure eventual convergence to a goal state.

\begin{figure}[bt]
\includegraphics[height=5.2cm]{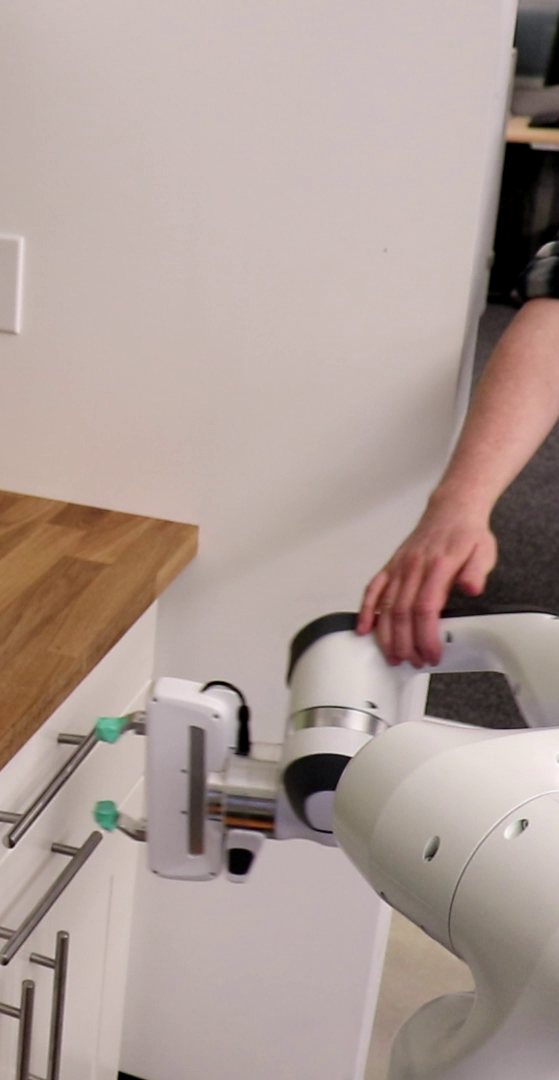}\hfill%
\includegraphics[height=5.2cm]{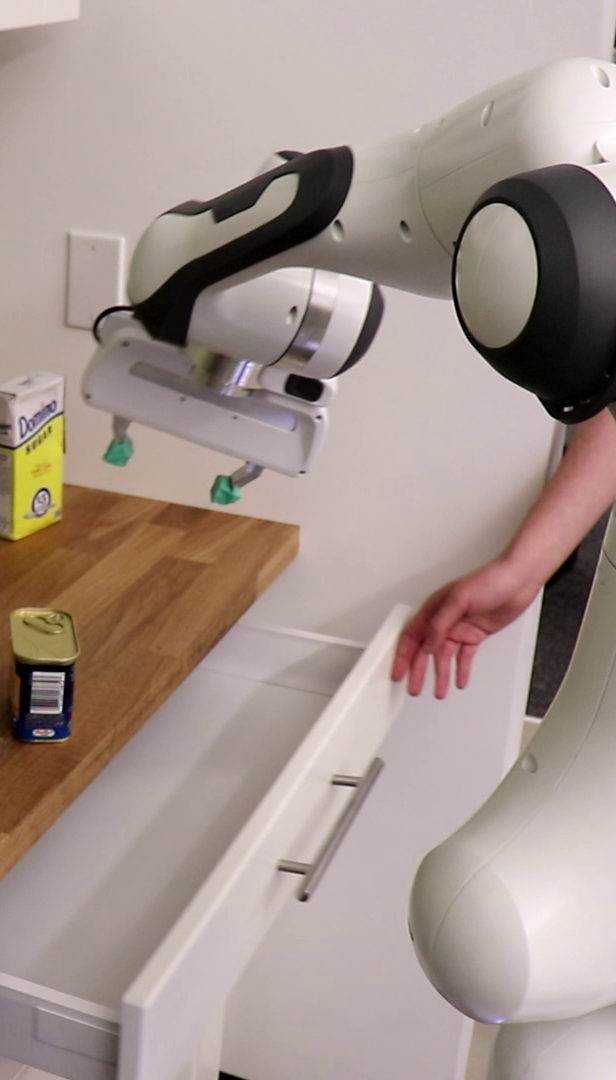}\hfill%
\includegraphics[height=5.2cm]{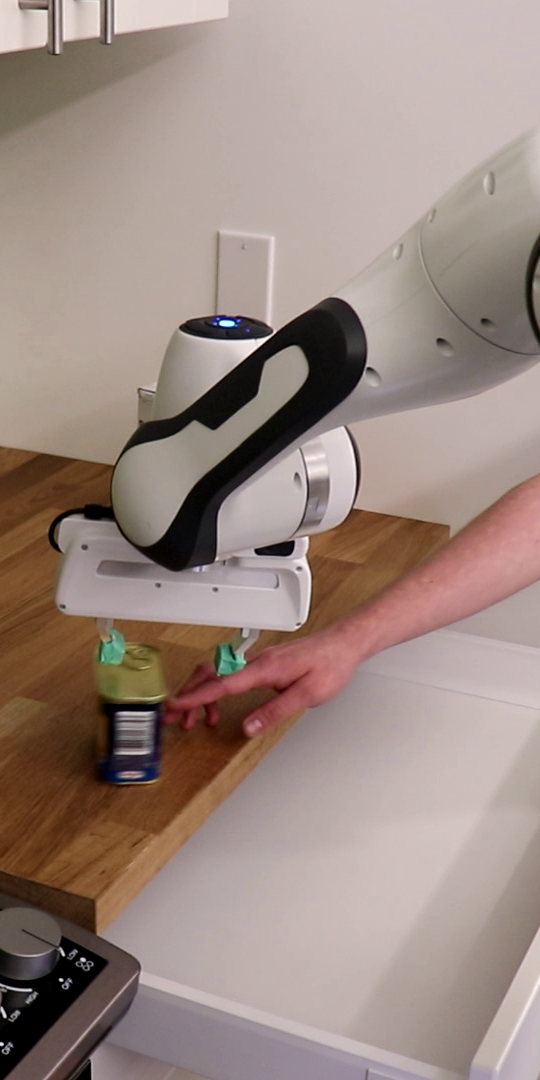}%
\caption{A Robust Logical-Dynamical System (RLDS) can represent reactive and robust behaviors. Here, our system successfully opens a drawer and places a can inside it despite various unforeseen challenges including sensor noise and interference from a human.}
\label{fig:cover}
\vskip -0.5cm
\end{figure}

Take, for example, the simple task of putting away a can of spam in a drawer (Fig.~\ref{fig:cover}). Our robot must open the drawer, pick up a can of soup, place it in the drawer, and close the drawer. 
In an FSM, we would see that nearly every one of these steps can connect to every other step under the right conditions. If the can falls out of the robot's hand, it must attempt to re-grasp it. If someone walks by and slams the drawer shut, the robot must put the can down and re-open it. If that same person walks by and puts the can away, the robot is finished, and doesn't need to open the drawer. All of these conditions make manual task creation fairly complicated, and also make it much harder to compose different tasks and sub-tasks in an intelligent way.
We see many of the same issues in a Behavior Tree; while composition is easier, specifying internal structure and preconditions on each state is still fairly difficult.

Robust Logical-Dynamical Systems (RLDS) are often substantially simpler for a large class of commonplace reactive recovery. Rather than specifying the internal structure of a task, a plan simply lists operators -- high-level task states like ``open the drawer'' -- in order of importance (usually as measured by either sequential proximity to the goal condition or priority as an evasive reaction). Each operator is associated with sets of preconditions, run conditions, and effects, to govern when it is allowed to be executed. RLDS work well with smooth, reactive real-time motion generation tools such as RMPs~\cite{cheng2018rmpflow}; the underlying continuous behavior of each operator is designed to drive the system toward state-transitions which enable us to concretely prove guarantees of the combined system dynamics on convergence to the goal.

Our contributions are:
\begin{itemize}
    \item We present RLDS as a system for reliable task execution which implicitly creates robust recovery behaviors.
    \item We derive an algorithm to automatically compose multiple RLDS, and show how to use RLDS as a part of a simple task planning algorithm.
    \item We prove theoretical convergence and convergence rate results showing that under mild covering conditions on the enterability of the operators RLDS's are guaranteed to converge to the goal condition despite perturbations of bounded probability.
\end{itemize}

We show our system in action on both a simulated and a real-world household object manipulation task, where the robot needs to construct a plan that will pick up one or more objects and place them in a drawer, as shown in Fig.~\ref{fig:cover}.

\section{Related Work}

McGann et al.~\cite{mcgann2009model} divide approaches into three categories: reactive systems, three-layer approaches, and planning-centric approaches. Reactive approaches involve encoding all behavior as a part of the underlying Finite State Machine (FSM) or Behavior Tree (BT). Examples include the controller sequencing~\cite{burridge1999sequential}, hybrid automatons~\cite{egerstedt2000behavior,eppner2016lessons}, SMACH~\cite{bohren2010smach}, and many BT implementations~\cite{paxton2017costar,paxton2018evaluating,hannaford2018behavior,colledanchise2018behavior}. Three-layer approaches divide between high-level deliberative planning, a sequencer in the middle, and low-level behavior, e.g.~\cite{gat1998three}. Planning-centric approaches include~\cite{muscettola2002idea,frank2003constraint,bresina2005activity,mcgann2008deliberative}. Often, for real-world systems, these still rely heavily on manually defined behaviors and hierarchies to achieve robust, reactive behavior, which raises the question of how best to define such behaviors and combine them.

Behavior Trees (BTs) have proven to be a powerful framework for specifying complex behaviors~\cite{colledanchise2018behavior}.
They have been used in medical~\cite{hannaford2018behavior} and industrial applications~\cite{paxton2017costar}. They are user-friendly~\cite{paxton2018evaluating}, with strong analogies to programming languages in their structure~\cite{colledanchise2016behavior}. One of the chief advantages of BTs is that all program behavior is determined by an internal logical state, which means that trees can be easily combined with one another to get robust behavior without a large amount of manual tuning. This is an important characteristic we retain in our system as well: 
conditions are continuously evaluated to determine which actions should be executed~\cite{colledanchise2016behavior}. However, the RLDS completely abstracts out internal details of the task plan, allowing us to specify problems purely in terms of goals and sets of operators as in PDDL~\cite{ghallab1998pddl}.

Linear Temporal Logic (LTL) is similarly a way of specifying complex task constraints~\cite{plaku2016motion}. Creating behaviors which satisfy these constraints can be difficult, however.

Another common way of specifying planning problems is via the Planning Domain Definition Language~\cite{ghallab1998pddl}, which has a very similar structure to our own problem definitions.
BTs commonly use preconditions similar to those placed on our operators to achieve complex behavior~\cite{paxton2017costar,colledanchise2018behavior}, and have been extended in the past to add PDDL-style preconditions and effects for the purposes of planning~\cite{rovida2017extended}.

We also see an analogy to Hierarchical Task Networks~\cite{georgievski2014overview}, which are a framework with the same representative power as PDDL/STRIPS but with hierarchical decomposition used to decrease planning complexity. These are one of the most commonly used system in practice, and
we also use this sort of hierarchical decomposition to enable code reuse and to simplify the correct sequencing constraints.

Task and Motion Planning (TAMP), e.g.~\cite{toussaint2015logic}, is a field that concerns itself with reasoning both about motion about logical constraints. TAMP approaches can solve very complex problems, but aren't designed to generate reactive behavior and aren't concerned with plan composition or plan extensibility. They
generally require online re-planning for reactivity, but their computational complexity can make that intractable.
We aim to avoid replanning in all but the most egregious cases, and instead rely on activating a succession of reactive low-level controllers at the correct times that could in principle include some planning.
In the future, each RLDS could be be constructed from a plan generated by a more computationally intensive TAMP solver; in that context it can be viewed as a way to robustly execute task plans leveraging reactions and systematic plan-operator re-entry with provable goal convergence guarantees.

Also of note, this form of hierarchical, reactive behavior has recently shown up in machine learning methods as well. Neural Task Programming, for example, hierarchically evaluates policies in order to reproduce a task performance in a new environment from a single video demonstration~\cite{xu2018neural}. In the future conditions and policies for an RLDS could be learned.

\section{Overview of Approach}\label{sec:overview}

Our goal is to describe a sequence of tasks in such a way that our system can automatically generate a robust reactive behavior to handle its execution. Intuitively, each task should have an associated logical condition describing whether it can be run (its runnable condition), and its goal should be to push the system toward the runnable condition of the next state. Moreover, these logical conditions should be Markov in the sense that we can classify whether a given state can be run independent of whether we know the history of states that have been run before, and they should (ideally) have a covering property meaning that we can always enter and run at least one of the tasks.

The Markov property enables the system to run any task whose runnable condition is met independent of whether it enters into that logical state via a controlled transition from successful execution of the preceding task or through some random perturbation from either external factors or the execution of reactions. Additionally, it enables us to implement a form of priority on the tasks. Multiple tasks may be runnable simultaneously, so we establish a convention that the most downstream of those takes precedence. If the entire sequence drives the system toward the goal, this convention implements a form of shortcutting since the most downstream task is closest to the goal. However, we may also define additional tasks downstream even from the goal state which act as reactions or evasive maneuvers. Those are always run with highest priority if needed.

RLDS's are agnostic to the specifics of the underlying behavior generation technique. In many cases, we can choose from a collection of low-level policies (e.g. implemented as RMPs), use specialized local feedback policies trained by Reinforcement Learning (RL) or Imitation Learning (IL), or even use guidance from motion plans computed by TAMP a solver used to generate the plan on which the RLDS is constructed. We just need to be able to characterize the behavior of the underlying policy to bound the probability of the policy resulting in a logical state transition, as described in the next section.

Unlike some prior work~\cite{toussaint2015logic}, we don't worry about the ``motion planning'' part of the problem. 
The idea is that we have a lot of specific policies that can do different things, and need to intelligently switch between them to get strong behavior. In theory, each of these policies can be trained on a very narrow set of conditions, e.g. opening or closing a drawer, opening a can, turning a knob, etc. We also assume that our low-level control policies have their own convergence guarantees.


In general, plans are most easily described as a chain of very general logical operators, which doesn't innately describe the full Markov logical state summarizing the result of having executed part of the chain. Sec.~\ref{sec:planning} details how to transform a plan into an RLDS using a simple algorithm that propagates conditions backward through each operator in the plan starting from the goal to generating a least-constraining set of logical conditions that fully describe the required Markov logical states.

\section{Formulation}\label{sec:approach}
Let $x\in\mathcal{X}$ denote the continuous state of the robot and world observable by the perception system, and denote the logical state by $l\in\mathcal{L}$. We can represent $l$ as a vector of binary values giving the truth value of a set of all groundings of logical predicates $\rho_i$. We denote the logical predicates generically as $\rho_i(\tau_1, \dots, \tau_k | x) \in \{0, 1\}$, where $\tau_j$ are their associated terms. I.e. given a continuous state $x$ the predicate takes on a truth value for each valid combination of terms (the predicate grounding); collecting those truth values up across all grounded predicates gives the logical state $l\in\mathcal{L}$.



For example, one predicate used in our manipulation case study from Sec.~\ref{sec:case-study} is 
\texttt{is\_attached\_to(robot\_part, object)}. The terms
\texttt{robot\_part} and \texttt{object} might be grounded by \texttt{end\_effector} and \texttt{sugar\_box}, respectively, giving the grounded predicate \texttt{is\_attached\_to(end\_effector, sugar\_box)} a particular binary truth value. This grounded predicate evaluates to $1$ when the end-effector is holding the sugar box and to $0$ otherwise. 


We denote a logical condition by $L$. Each logical condition can be represented as a logical function of the available predicates. These logical conditions in practice are often implemented simply as a conjunction of a list of predicates. For each logical condition, there is a specific set of continuous states $x$ and associated logical states $l$ which make it true. Denoting the full joint continuous-logical state space (referred to typically as simply the {\it logical state space}) as $\mathcal{S} = \mathcal{X}\times\mathcal{L}$, we often also overload the notation $L$ to denote the set of all logical state vectors that render the logical condition true: $\mathcal{L} \equiv \{s\in\mathcal{S} | L(s) = 1\}$. For instance, if we have two logical conditions $L_1$ and $L_2$ such that $L_1\Rightarrow L_2$, we can say $L_1\subset L_2$. 

\subsection{Robust Logical-Dynamical Chains}

The simplest form of an RLDS is a Robust Logical-Dynamical Chain (RLDC); all RLDS's discussed in this paper (including those produced by hierarchical combination or automatically generated by a planner) can be reduced to a chain. This section presents a formal mathematical construction of the RLDC which is used in the analysis of Section~\ref{sec:TheoreticalAnalysis}. 

The fundamental element of an RLDC is termed an {\it operator} $o\in\mathcal{O}$. We can view each operator as a tuple $o=(L_P, L_R, L_E, \pi)$ of logical conditions and an associated policy $\pi$. $L_P$ is the entry condition defining whether the operator can be entered (if $L_P$ is true, we say the operator is ``enterable''), $L_R$ is the run condition defining whether operator can be continue to be run if it has already been entered (if true, it's said to be ``runnable''), and $L_E$ defines the expected logical condition that results from running the operator.
The distinction between $L_P$ and $L_R$ can be used to implement robust entry into a state to prevent oscillations resulting from stochastic dynamics, allowing the system, for instance, to optionally attempt a maneuver before resetting. In many practical cases $L_P = L_R$.



The policy is defined as $\pi:\mathcal{S}\rightarrow\mathcal{U}$, where $\mathcal{U}$ is the set of control actions. Each $pi$ pushes the system from $L_R$ toward satisfying the operator's effects $L_E$. In conjunction with the system's logical state dynamics $f:\mathcal{S}\times\mathcal{U}\rightarrow\mathcal{S}$, which we require to be Markov in the standard sense,\footnote{Note that the underlying continuous state space typically already has Markov dynamics, but the Markov requirement dictates that the logical state be expressive enough to also maintain this Markov property.} the policy creates a natural system flow
$s_{t+1}=f(s_t,\pi_i(s_t))$ through the logical state space which we can characterize concretely in terms of $L_E$. 

We say a state $s$ is {\it feasible under operator $o_i$}
if it satisfies the runnable condition $L_R^i$. And we call a sequence of feasible states generated by the underlying policy $\pi_i$ a {\it feasible sequence}. Likewise, we say a state sequence {\it terminates} if either an infeasible state is reached (resulting in an {\it infeasible sequence}) or if $L_E^i$ becomes satisfied. Note that there are potentially many policies that can implement the same logical behavior; in many ways the logical behavior is agnostic to policy choice aside from differences in overall convergence properties characterized below. 

We are now equipped to define the RLDC. An RLDC is a sequence of operators  $\vec{o} = \big((o_1),\ldots,(o_N)\big)$ for which the following local chaining properties hold between pairs of operators:
\begin{align} \label{eqn:ChainingCondition}
    L_E^i \Rightarrow L_P^{i+1} \Rightarrow L_R^{i+1}.
\end{align}
These properties state that the effect of operator $o_i$ implies enterability into the next operator $o_{i+1}$, which in turn implies runnability of that operator. Additionally, we call an RLDC {\it complete} if $\cup_i L_P^i=\mathcal{S}$. Note that an operator with an associated policy can be viewed as a {\it task}, so an RLDC can be seen as a sequence of tasks.

At all times the system always enters a downstream operator if possible; additionally, if the logical state becomes infeasible under the current operator for any reason the system transitions into the most downstream enterable operator (which usually ends up being an upstream operator). This gives an implicit priority to the operators, with downstream operators (e.g. those closer to the goal) taking precedence. 

If each $\pi_i$ generates a feasible sequence until achieving $L_E^i$ with probability 1, then the chaining condition of Equation~\ref{eqn:ChainingCondition} implies trivially that the RLDC converges to $L_G$ in $N$ transitions. In general, however, system perturbations, perceptual uncertainty resolution, or stochastic dynamics commonly produce in infeasible sequences resulting in spurious uncontrolled transitions. Section~\ref{sec:TheoreticalAnalysis} analyzes convergence and convergence rate of these chains despite uncontrolled transitions.




\subsection{Theoretical Analysis: Convergence of Chains to $L_G$} \label{sec:TheoreticalAnalysis}

Intuitively, the underlying dynamics of each $\pi_i$ drives the system toward $L_E^i$, resulting in successful forward transitions toward the goal. If the likelihood of backward transitions can be bounded we should be able to prove convergence to the goal. This section makes that observation concrete.

We say $o_i$ induces a {\it controlled transition} with probability $p_i$ if with that probability it generates a feasible sequence terminating in $L_E^i$ satisfied. Similarly, sequences terminating in infeasibility are said to generate {\it uncontrolled transitions}.

\begin{theorem}{(Convergence)} Robust logical-dynamical chains achieve $L_G$ with probability 1 if each logical state $i$ induces a controlled transition with probability $p_i \geq p > 0$, converging exponentially in the number of uncontrolled transitions $k$. Moreover, the system takes an expected number of transitions $T$ upper bounded by $E[T]\leq \frac{N}{p^N}$.
\end{theorem}
\begin{proof}
For simplicity, we analyze the case where all $p_i=p$. The probability of never experiencing an uncontrolled transition and reaching the goal in a single run through the chain is bounded below by $p_g = p^N$. Similarly, the probability of experiencing exactly $k$ uncontrolled transitions (i.e. not reaching the goal $k$ times) is bounded by $p_k = (1-p^N)^k p^N$ (wherein the uncontrolled transition is all the way back to the start). Therefore, the probability of successfully reaching the goal with at most $k$ uncontrolled transitions is 
\begin{align*}
P_k 
&= \sum_{i=0}^k p_i = p^N\sum_{i=0}^k(1-p^N)^i \\
&= p^N\left(\frac{1-(1-p^N)^{k+1}}{1-(1-p^N)}\right) = 1-\gamma^{k+1},
\end{align*}
for $\gamma = 1-p^N$. This probability $P_k$ converges exponentially to 1 as $k \rightarrow \infty$. 

Moreover, if the system experiences exactly $k$ uncontrolled transitions before succeeding, the largest number of steps it can take is $k(N-1) + N \leq (k+1)N$. Therefore, the expected number of transitions is 
\begin{align*}
E[T] \leq  \sum_{k=0}^\infty (k+1) N P_k = Np^N\sum_{k=0}^\infty (k+1)\gamma^k,
\end{align*}
where again $\gamma = 1-p^N$. Noting that
\begin{align*}
&\sum_{k=0}^\infty (k+1)\gamma^k
= \sum_{k=0}^\infty \gamma^k + \sum_{k=1}^\infty \gamma^k + \sum_{k=2}^\infty \gamma^k + \cdots \\
&\ \ \ \ \ = \sum_{k=0}^\infty \gamma^k\left(1 + \gamma + \gamma^2 + \cdots\right)
= \left(\frac{1}{1-\gamma}\right)^2,
\end{align*}
the expectation reduces to
\begin{align*}
E[T] \leq \frac{Np^N}{(1-(1-p^N))^2} = \frac{N}{p^N}.
\end{align*}
\end{proof}

As an example, if each transition is successful with probability $.9$ and there are $N=5$ states, we would expect the number of transitions to be $E[T] \leq \frac{5}{.9^5} \approx 8.47 < 9$. The expected (upper bound) inflation factor is given by $\nu=p^{-N}$. For $p=.95$, $N=10$ we get $1.67$; for $p=.8$, $N=4$ we get $2.44$.

\subsection{Execution}

Alg.~\ref{alg:execution} describes execution of a Robust Logical-Dynamical Chain.
We follow the paradigm used in Behavior Trees~\cite{paxton2017costar,colledanchise2018behavior} and in hierarchical learned models (e.g.~\cite{paxton2017combining,huang2018neural,xu2018neural} whereby hierarchical decisions are made constantly and in nearly real-time.
While plan construction takes place offline, we assume that inferring the specific policy to execute occurs at close to real time (e.g. at 10-30 hz). The core of the algorithm, at every time $t$, checks multiple $o_i \in \vec{o}$ to determine the highest-priority operator whose preconditions are met, assigns this to the current operator $o_t$.


\begin{algorithm}[bt]
\caption{Execution of a Robust Logical-Dynamical Chain that achieves goal $L_G$}\label{alg:execution}
\begin{algorithmic}
\Function{Execute}{Operators $\vec{o}$, Goal $L_G$}
\State $o_{0} = \emptyset, t=0, l_0=L(x)$
\While{not $L_G \subseteq l_t$}
\State $t = t + 1$
\State $x_t \leftarrow W$ // get latest continuous world state
\State $l_t \leftarrow L(x)$ // compute predicates on current state
\For{$i \in (\text{length}(\vec{o}), \dots, 1)$}
    \If{$o_i \not= o_{t-1}$ and $L_P^i \subseteq l_i$}
        \State $o_t = o_i$; \textbf{break}
    \ElsIf{$o_{t-1} = o_{i}$ and $L_R^i \subseteq l_i$}
        \State $o_t = o_{t-1}$; \textbf{break}
    \EndIf
\EndFor
\State Step $\pi_i(x_t)$
\EndWhile
\EndFunction
\end{algorithmic}
\end{algorithm}

\subsection{Equivalence to Behavior Trees}

\begin{figure*}[bt]
\centering
\includegraphics[width=1.4\columnwidth]{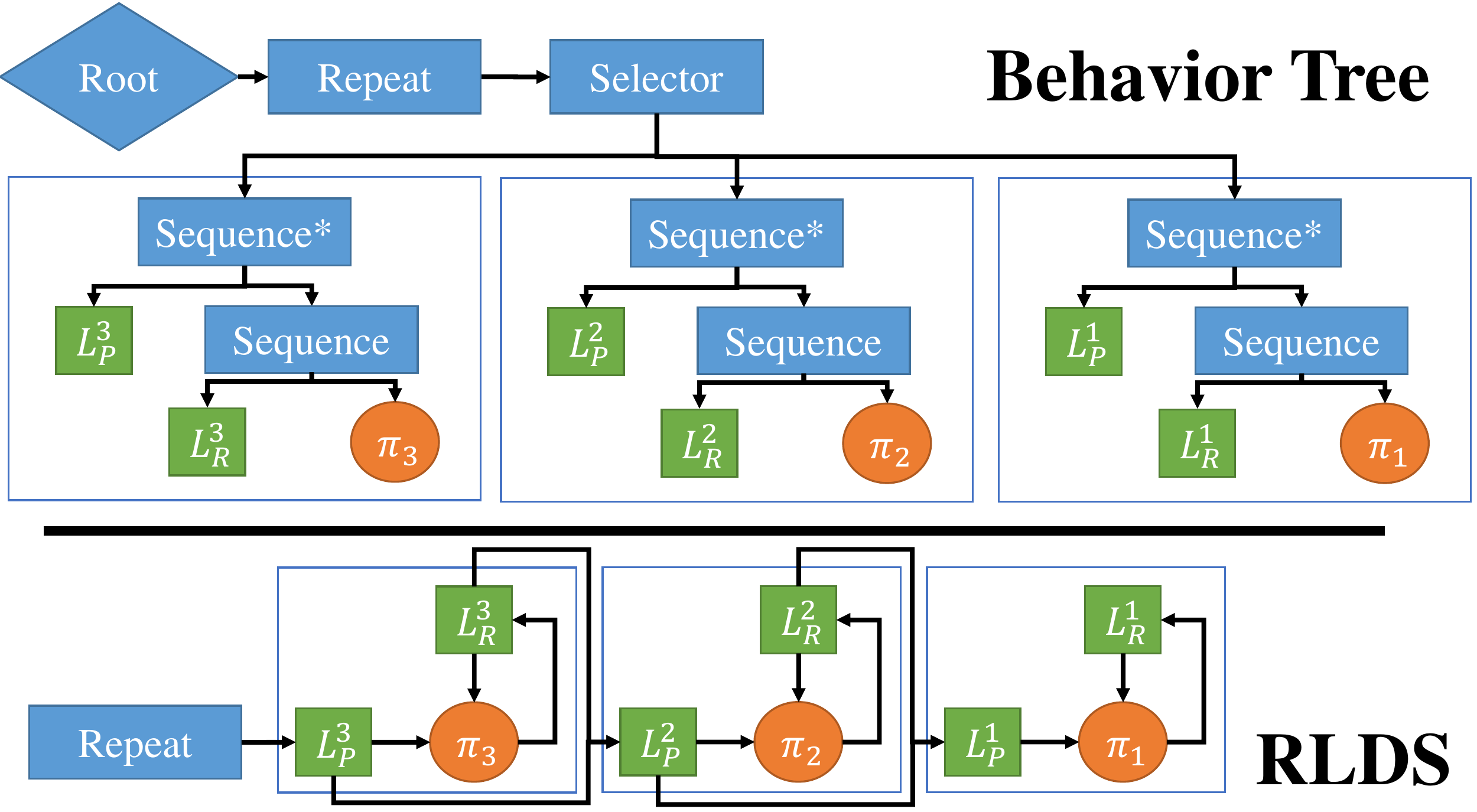}
\caption{Any given RLDS can be modeled as an equivalent Behavior Tree (BT). Likewise, RLDS can be composed in the same manner as BTs.
The Behavior Tree implementation here uses some internal nodes with memory (Sequence*) to denote when a node has passed its preconditions. The RLDS checks entry or run conditions as necessary. The differentiation between run conditions and preconditions is very important for robust execution without oscillations between states.}
\label{fig:bt}
\end{figure*}

This formulation of the RLDS, while very convenient for concisely programming reactive behaviors for robust execution, can be easily shown to be as expressive as Behavior Trees, a common framework for representing hierarchical tasks that has been increasingly popular both in robotics and in other fields such as video games~\cite{paxton2017costar,colledanchise2018behavior}. In a Behavior Tree, flow of operations is determined by internal logical nodes, computed based on the output of child nodes.  We could equivalently express any of these BTs as an RLDS by adding logical predicates to describe internal operations of the tree. Similarly, program flow in an RLDS could be easily described as a BT. Fig.~\ref{fig:bt} shows a BT (top) and an equivalent RLDS (bottom).

The RLDS has a couple additional features, including the addition of the run conditions automating state transitions implementing reactions and recovery behaviors, which while possible in a BT are often cumbersome to program and error prone. Relative to Robust Logical-Dynamical Systems as they have been defined so far, Behavior Trees do have some advantages in that they can be easily composed to produce more complicated behavior. Section~\ref{sec:planning} introduces an analogous form of composibility for RLDS's which greatly improves reusability of subcomponents.

\section{Plan Composition}\label{sec:planning}

While Alg.~\ref{alg:execution} allows us to easily and robustly execute an arbitrary RLDS, it leaves us with one major problem: the conditions must be exhaustively specified, and the solutions are not easy to combine or re-use in new contexts. Fortunately, we have solutions to these problems that show how multiple RLDS can be combined hierarchically or sequentially through different algorithms.

The point of including the effects $L_E(\pi)$ is that they give us two specific advantages. First, we can detect failures, i.e. when a particular policy $\pi$ was unable to reach its goal condition after some amount of time
Second, we can use these to compose plans and compute \emph{implicit conditions} guarding when these policies can be entered and executed in order to guarantee that we will eventually arrive at our logical goal $L_G$.

\subsection{Hierarchical Composition}

Hierarchical composition is a useful capability for any task representation, as it allows for code re-use and increased generalization, and simplified debugging, making programming substantially simpler. One of the chief advantages of BTs over systems like FSMs is how easy it is to compose two BTs~\cite{colledanchise2018behavior}: composition is determined by the logical nodes and structure surrounding the tree, which does not itself need to change or add any new connections.
The general idea remains the same in an RLDS, though for an RLDS the additional logical structure that governs execution is instead given through the predicate sets $L_P$, $L_R$, and $L_E$. 

Given a plan $a = \left\{o_1, \dots, o_n\right\}$, we can create a new operator $o_a$, where $L_E^a$ is the expected cumulative effects of $o_1, \dots, o_n$, and $L_P^a$ and $L_R^a$ are both empty. To create specialized behaviors, we can simply add predicates to these two sets. 
For example, imagine a task where we want to attempt to open a cabinet until that has been accomplished, then pick up an object, as explored in Sec.~\ref{sec:case-study}. We can define a re-usable \texttt{pickup} RLDS as $o_{pickup}$, and then add the {\tt{cabinet\_is\_open}} predicate to $L_P^{pickup}$ and $L_R^{pickup}$ to ensure it is executed at the correct time in the task plan. Then we can specify the plan:
\[
    P = \left\{o_{open}, o_{pickup}, p_{place} \right\}
\]

\subsection{Sequential Composition}

The logical sets corresponding to each of our different operators are not guaranteed to match up to enable them to be arbitrarily sequenced to execute tasks -- in fact, they most likely will not. General-purpose operators have conditions that will almost always be met, leading to oscillations in task execution and a failure to reach a goal state. 

In practice, however, there are certain \textit{implicit} conditions added to our operators by the task plan and the goal. In order to reach that goal, we need to enforce that each operator actually takes us into the precondition space of the remaining sequence of actions.
The implicit conditions $L_I$ are the set of predicates that are set by preceding actions then used by future actions or are a part of the goal. This is enough to restrict the set of logical state trajectories to just those that allow for the task to be accomplished.

For example, take a subtask $a$, such as ``pick up the soup can.'' This includes reactive steps to move to a pre-grasp, cage the object, close the gripper, and lift the object up. In order to combine it with the full plan $\vec{o} = \text{``put the soup in the top drawer''}$ we need to add the necessary boundary conditions that differentiate it from the other steps in the plan, such as the fact that the drawer should be open before we pick up an object.

These conditions can be computed by a simple recursive algorithm that works back from $L_G$ to compute implicit conditions defining the necessary execution order. The core idea is that we want to find predicates $p$ that are required later in a plan, either by $L_G$ or some $L_P$. We propagate the set of implicitly required conditions back from the end of the plan, removing entries in the set when they are expected results of an operator, e.g. when $p \in L_E$, and adding more entries as we see new required preconditions of subsequent operators. In addition, we can create the most general set of implicit conditions by only adding implicit conditions when they are effects of a prior operator in the plan.

\begin{algorithm}[bt]
\caption{Computation of Implicit Conditions}
\label{alg:implicit}
\begin{algorithmic}
\Function{GetImplicitConditions}{$P$, $L_G$}
\State \textbf{given:} Plan $P$, goal $L_G$
\State // backwards pass: compute implicit conditions
\State $N = \text{length}(P)$
\State $L_P^{N+1} = L_G$, $L_I^{N+1} = \emptyset$
\For{$i \in N, \dots, 1$}
    \State $L_{I}^i =  \left\{\rho \in L_P^{i+1} \cup L_{I}^{i+1} \forall \rho \not\in L_E^i \right\}$
\EndFor
\State \Return $L_I$
\EndFunction
\end{algorithmic}
\end{algorithm}

The algorithm for computing the set of implicit conditions on plan execution is given in Alg.~\ref{alg:implicit}.
Once $L_I$ has been computed, we can simply state that for all $o_i \in P$, $L_P^i = L_P^i \cup L_I^i$ and $L_R^i = L_R^i \cup L_I^i$.

\subsection{Planning}

Finally, we discuss how these two components can be combined to automatically generate task plans for execution.
Given that a set of preconditions, effects, and optionally a hierarchical decomposition is given for any particular RLDS domain, we could use any of a variety of STRIPS-style planners~\cite{fikes1971strips} to solve the problem. In addition, an RLDS can easily be phrased as a PDDL planning domain with a goal condition; the run-condition set $L_R$ can generally be ignored when planning.

For our purposes, we use a simple greedy algorithm that is effective in generating solutions in our domain. Given a node in a search tree defined by the logical state $l$ with depth $D(l)$, we apply all possible operators and search according to a simple heuristic $h(l)$, where:
\[
h(l) = D(l) + \|l - \{\rho \forall \rho \in L_G\}\|_1
\]
\noindent to encourage the planner to quickly find short paths.
The complete planning algorithm is given in Alg.~\ref{alg:planning}. Here, the {\sc{Update}} function looks up the search node corresponding to logical state $L$ and updates the back-pointer to the best parent node, as per $A^*$ search, and the {\sc{Backup}} operator takes a logical state $L$ and finds its parents from the tree search.

\begin{algorithm}[bt]
\caption{Simple Task Planning and Execution with Robust Logical-Dynamical Systems}\label{alg:planning}
\begin{algorithmic}
\Function{PlanAndExecute}{$L_0$, $L_G$, $\mathcal{O}$}
\State \textbf{Given:} initial state $L_0$, goal $L_G$, operators $\mathcal{O}$
\State $q =$ \Call{PriorityQueue}{}()
\State Add $L_0$ to $q$
\While{$q \not= \emptyset$}
    \State $L =$ \Call{pop}{$q$} // get lowest cost logical state
    \If{$L_G \subseteq L$}
        \State \textbf{break}
    \EndIf
    \For{$o \in \mathcal{O}$}
        \If{$l \in L_P^o$}
            \State $l' = L_E^o(l)$ // apply effects from operator
            \If{$l' \in T$}
                \State \Call{UpdateNode}{$l'$, parent=$l$}
            \Else
                \State \Call{Push}{$q$, $l'$} // add to queue
            \EndIf
        \EndIf
    \EndFor
\EndWhile
\State $\vec{o} =$ \Call{Backup}{$L$}
\State $L_{I}$ = \Call{GetImplicitConditions}{$\vec{o}$, $L_G$}
\State \Call{Execute}{$\vec{o}$, $L_{I}$}
\EndFunction
\end{algorithmic}
\end{algorithm}

Future work will look at using more advanced planning approaches in place of the simple tree search used here, including the integration of full-scale PDDL planners like FastDownward~\cite{helmert2006fast} for planning of more complex longer-horizon tasks.

\section{Case Study}\label{sec:case-study}

We explore our proposed task representation in the form of a kitchen manipulation system. This robot is expected to pick up objects and move them around in human environments, which necessitates some amount of reactivity.

\subsection{Example Operators}

\begin{figure}[t]
\includegraphics[width=\columnwidth]{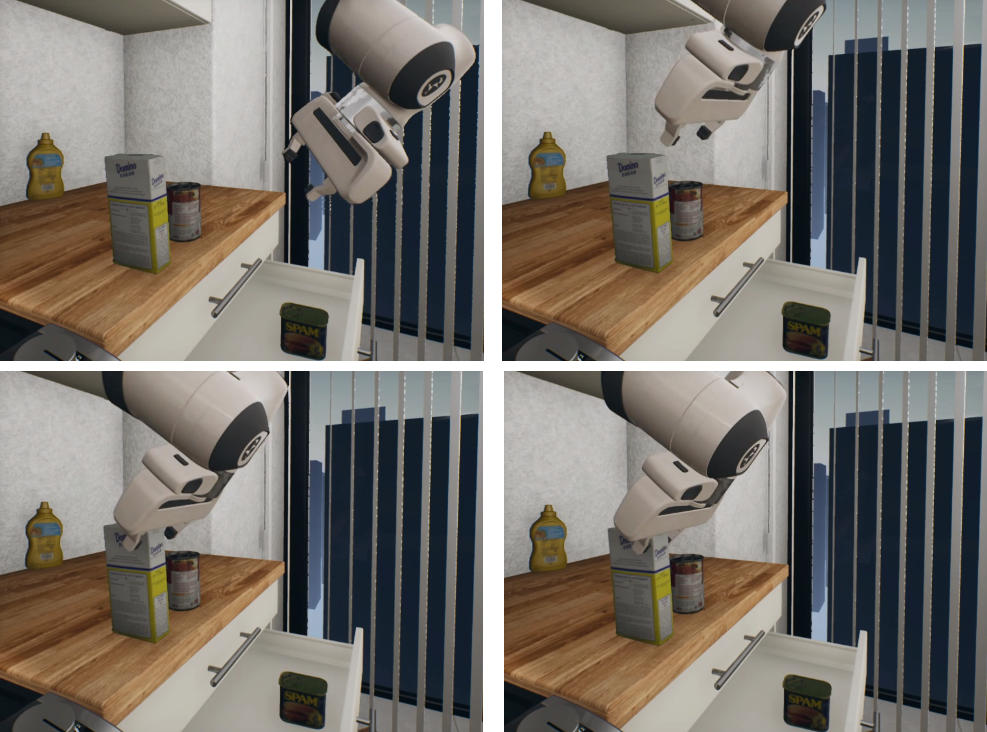}
\caption{An example of grasping behavior from our case study. This corresponds to three low-level operators, each with their own preconditions, run conditions, and effects: \texttt{approach}, \texttt{cage}, and \texttt{grasp}. This sub-sequence can be re-used in many different parts of our task plan.}
\label{fig:grasping}
\vskip -0.5cm
\end{figure}

We will discuss three specific operators implemented in our kitchen manipulation case study: \texttt{approach}, \texttt{cage}, \texttt{grasp}. Each of these operators can be applied to any object defined in the kitchen manipulation domain. These constitute an RLDS whose purpose is to grasp an arbitrary object. The grasping behavior is shown in Fig.~\ref{fig:grasping}.
There are correspondingly three crucial predicates in this RLDS:
\begin{itemize}
    \item \texttt{in\_approach\_region(robot, obj)}, \item \texttt{around\_obj(robot, obj)}, and
    \item \texttt{is\_attached\_to(robot, obj)}.
\end{itemize}
\noindent The \texttt{in\_approach\_region(robot, obj)} predicate defines whether or not the robot is on track to complete a grasp. It defines a cylindrical volume along a line between a standoff position and a known good grasp. \texttt{around\_obj(robot, obj)} defines an error margin around this grasp position, and \texttt{is\_attached\_to(robot, obj)} is true after we have closed our gripper around an object.

The three policies $\pi_{approach}, \pi_{cage}, \pi_{grasp}$ will move us between the continuous state space corresponding to each of these predicates.
Ideally, for an RLDS, all policies will themselves be reactive and able to respond very quickly when objects move or change positions.
Our policies are implemented as Riemannian Motion Policies~\cite{cheng2018rmpflow}.

Policies need to be somewhat flexible, so that they can support multiple grasps and multiple object orientations. We implement a lookup table which will lookup a 6-DOF goal pose $T_g \in SE(3)$ from a list of poses based on some user-provided cost function. With current pose $T_t = (R_t, p_t)$ and goal $T_g = (R_g, p_g)$, where $T_t \in SE(3)$ is the current end-effector pose, with $R$ as the rotation matrix and $p$ as the translation component. Given $R_{t,g} = R_{g}^{-1} R_t = \left(e, \theta\right)$, we choose:
\[
    T_g = \argmin_{T_g} \lambda_p \|p_g - p_t\|_2 + \lambda_R | \theta |
\]
\noindent where $\lambda_p$ and $\lambda_R$ are defined by the problem domain. We also use this same function to determine if the arm is currently in the approach volume, substituting $T_g$ for the closest point along a line between the standoff position and the grasp position. We define motion policies that can robustly move us to any of the pose offsets specified in the domain.

\subsection{Simulation Experiments}

We perform robotic manipulation experiments in a simulated kitchen environment.
The goal is to take a random set of objects and place them inside a randomly chosen drawer in the kitchen domain described above. Policies were manually defined to move to any of a number of user demonstrated grasp locations on the objects.

We use three objects from the YCB~dataset~\cite{calli2017yale}: a can of spam, a sugar box, and a tomato soup can. All of these have to be grasped in a different way and placed into one of two drawers in a kitchen cabinet. Initial object positions are randomized within the reachable workspace of the robot.

To show the benefits of reactivity we compare three different execution strategies: \textit{linear execution} with and without replanning, and our proposed reactive execution algorithm.

\paragraph{Linear execution w/ replanning} We call our planning algorithm once to generate a plan as a sequence of operators, and then compute implicit conditions on the plan as per Alg.~\ref{alg:planning}. This plan can only be executed in order: we check $L_P$ for the current operator $o_i$ and for the subsequent operator $o_{i+1}$. If $L_P^{i+1}$ is met, we move on to the next operator and execute it. If it is not met but $L_R^{i}$ is, we continue executing the current operator. This does not allow the system to repeat sequences of actions, or to adapt to noisy interactions with the world.

\paragraph{Linear execution w/o replanning} If neither $L_P^{i+1}$ nor $L_R^i$ are true at the current timestep, we replan and recompute implicit conditions, then execute again, essentially re-running all of Alg.~\ref{alg:planning}. This case allows the robot to adapt, but is more computationally intensive. To make planning more efficient, we limited ourselves to a subset of all available operators.

\paragraph{Reactive execution} Finally, we demonstrate our results on a reactive plan as per Alg.~\ref{alg:planning}, executed via Alg.~\ref{alg:execution}.

We ran 10 trials, recording duration, success or failure, and number of replanning attempts.
We demonstrate reactivity through adversarial interference: 5 seconds after the robot opens a drawer the first time, the trial pauses and the drawer is closed a random amount. Then execution resumes, and the robot must determine how to handle its new environment.

\subsection{Results}

Table~\ref{tab:results} shows results from simulations. We see that replanning and reactive achieve similar performance on randomly-generated tasks, though replanning may take slightly longer, and that without either replanning or reactivity we could not handle stochastic interactions between the robot and objects in its environment.

The effect of stochastic simulation is particularly apparent in the poor performance of the non-reactive task model:~60\% of the trials fail, even without adversarial interference. As seen in Fig.~\ref{fig:failure}, This often occurs when a grasp ends up in a slightly different pose than expected, meaning that the conditions were not met precisely. In an RLDS, the system would automatically transition back to a known good state (such as movement to a standoff position) and retry the grasp. Another issue is that linear execution does not recognize when execution can jump ahead. The reactive models will instantly adapt if, for example, the can is accidentally knocked into the drawer.

\begin{table}[bt]
    \centering
    \begin{tabular}{l c c}
        \hline
         & Success Rate (\%) & Completion Time (s) \\
         \midrule[1pt]
         \textbf{Baseline} & & \\
         \hspace{1em}Linear Execution & 60\% & $56.11 \pm 8.69$  \\
         \hspace{1em}Linear with Replanning & 100\% & $56.69 \pm 7.32$ \\
         \hspace{1em}Reactive Plan & 100\% & $52.23 \pm 7.93$\\
         \midrule[1pt]
         \textbf{With Interference} & & \\
         \hspace{1em}Linear Execution & 0\% & n/a \\
         \hspace{1em}Linear with Replanning & 100\% & $83.84 \pm 14.10$ \\
         \hspace{1em}Reactive Plan & 100\% & $78.45 \pm 9.84$ \\
         \hline
    \end{tabular}
    \caption{Execution time of plans with random goals under different conditions. Reactive plans are more efficient than replanning and are robust to environmental variations.}
    \label{tab:results}
\end{table}

\begin{figure}[bt]
\centering
\includegraphics[width=0.9\columnwidth]{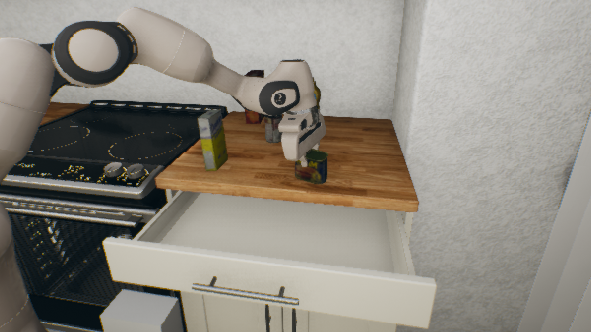}
\caption{Without reactivity, the robot cannot recover from bad approaches or handle the stochasticity of a realistic environment. In this case, the robot attempted to grab a can of spam, but came at a slightly bad angle. In the RLDS, it retries this grasp until it succeeds. In a more traditional system, this results in a failure.}
\label{fig:failure}
\vskip -0.5cm
\end{figure}

In our planning tests, we replanned 7 times out of 10 before adding any interference, and an average of 1.8 times per trial in the second set of experiments. In our simple case, this did not translate to a significant difference in execution time, but in more complex environments replanning would quickly become infeasible. It also allows for less natural, less responsive behavior. In the future, we could imagine much more complex tasks that use a mixture of both planning and RLDS execution to achieve superior performance on difficult real-world problems.

\subsection{Real-World Experiments}

\begin{figure*}
\centering
\includegraphics[width=2\columnwidth]{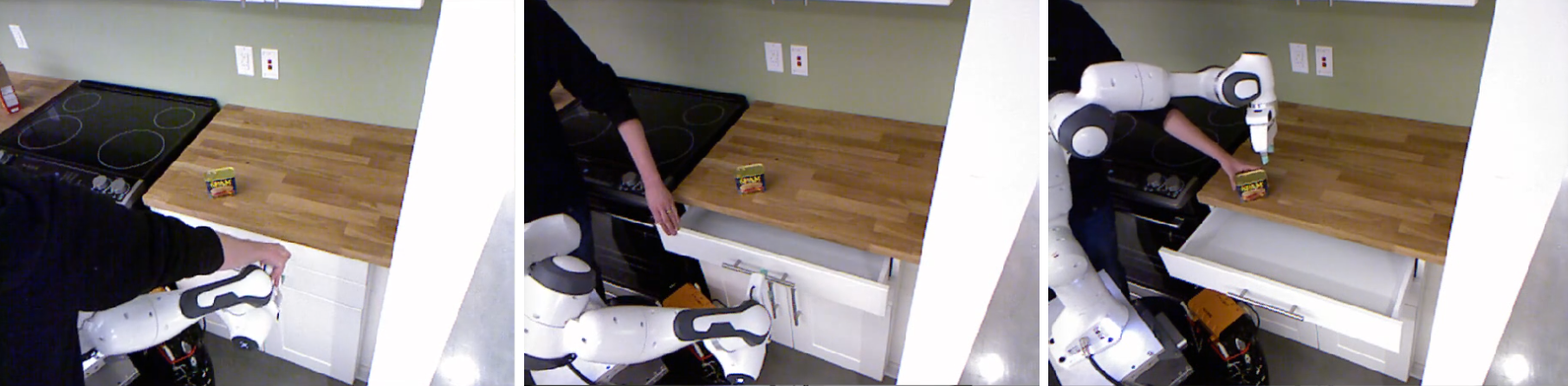}
\caption{In the real world experiments, a human pushed on the robot, shoved its end effector, interfered with the drawer, and moved the object around. The RLDS was able to smoothly adapt to changing circumstances and complete the task regardless.}
\label{fig:real-world}
\end{figure*}

We also tested our RLDS on a similar task in the real world version of the kitchen in Fig.~\ref{fig:grasping}, shown in Fig.~\ref{fig:real-world}. Here, we tested opening a drawer and placing a single YCB object (a can of spam) inside it. This was made more challenging by the use of stochastic perception: the robot needed both to estimate its current position relative to a camera, and to detect and track the object in order to pick it up and place it. This necessitates a more complex task plan.

We executed on a Franka Emika Panda 7-DoF~manipulator. The world state, including the position of the robot and the state of the cabinet, was estimated using DART~\cite{schmidt2014dart}. Objects were detected with PoseCNN~\cite{xiang2017posecnn}. We added new operators to retract the arm and to detect the objects and start tracking. When the arm nears objects, it will ``unfix'' its position in DART, and start jointly optimizing the robot's base position with the rest of the scene. This meant that robust execution was extremely important, and that the robot sometimes needed to back off from an attempt to grasp the object or drawer and retry the grasp before succeeding.

We also introduced an adversary during execution. A user shoved on the drawer to close it after the robot opened it, shoved the robot's hand as it attempted to grasp either the handle or the spam, or pushed the spam out of the way. The RLDS was able to recognize all of these failure cases and backtrack accordingly, retrying the necessary steps in order to complete its task. Experiments are shown in Fig.~\ref{fig:cover}, Fig.~\ref{fig:real-world}, and the accompanied video.\footnote{\url{https://youtu.be/l_8pzcRGztk}}

\section{Discussion and Conclusions}
We proposed a powerful, versatile framework that allows for the creation and composition of reactive robot task plans. The RLDS will automatically capture fallback behaviors and connections between different robot policies that need to be executed, and allow us to adapt to changes in the environment. RLDS can be constructed automatically, and are a useful, composable way to build realistic robot behaviors by capturing them as a series of preconditions, run conditions, and effects. In the future, we will use RLDS to describe a wider range of behavior and look into integration with task and motion planning. 

\bibliographystyle{IEEEtran}
\bibliography{task}

\end{document}